\documentclass{article}

     \usepackage[preprint,nonatbib]{neurips_2020}

\usepackage[utf8]{inputenc} 
\usepackage[T1]{fontenc}    
\usepackage{hyperref}       
\usepackage{url}            

\usepackage{booktabs}       
\usepackage{amsthm,amsmath,amsfonts}
\usepackage{mathtools}      
\usepackage{nicefrac}       
\usepackage{microtype}      
\usepackage{graphicx}
\usepackage{subfigure}
\newtheorem{theorem}{Theorem}

\newtheorem{remark}{Remark}
\newtheorem{definition}{Definition}[section]

\title{Framework for Designing Filters of Spectral Graph Convolutional Neural Networks in the Context of Regularization Theory}

\author{%
  Asif Salim, Sumitra. S
   \\
  Department of  Mathematics,\\
  Indian Institute of Space Science and Technology,\\
  Thiruvananthapuram, India\\
  \texttt{asifsalim.16@res.iist.ac.in, sumitra@iist.ac.in} \\
}

\begin{document}

\maketitle

\begin{abstract}
 Graph convolutional neural networks (GCNNs) have been widely used in graph learning. It has been observed that the smoothness functional on graphs can be defined in terms of the graph Laplacian. This fact points out in the direction of using Laplacian in deriving regularization operators on graphs and its consequent use with spectral GCNN filter designs. In this work, we explore the regularization properties of graph Laplacian and proposed a generalized framework for regularized filter designs in spectral GCNNs. We found that the filters used in many state-of-the-art GCNNs can be derived as a special case of the framework we developed. We designed new filters that are associated with well-defined regularization behavior and tested their performance on semi-supervised node classification tasks. Their performance was found to be superior to that of the other state-of-the-art techniques.
\end{abstract}

\section{Introduction}

Convolutional neural networks (CNNs) \cite{726791} have been applied to a wide range of problems in computer vision and speech processing for various tasks such as image classification/segmentation/detection and speech recognition/synthesis etc.  The success of CNNs as a powerful feature extractor for data in the Euclidean domain motivated researchers to extend the concepts to non-euclidean domains such as manifolds and graphs \cite{8100059}.   

The GCNN designs mainly fall into two categories namely spatial and spectral approaches. Spatial approaches directly define the convolution on the graph vertices. They are based on a message-passing mechanism \cite{gilmer2017neural} between the nodes of the graph. 
Examples are network for learning molecular finger prints \cite{duvenaud2015convolutional}, Molecular graph convolutions \cite{kearnes2016molecular}, GraphSage \cite{hamilton2017inductive}, Gated graph neural networks \cite{li2015gated}, graph attention networks \cite{velivckovic2017graph} etc. Spectral filter designs are based on the concepts of spectral graph theory \cite{chung1997spectral} and signal processing on graphs \cite{shuman2013emerging}. The graph fourier transform has been utilized to define graph convolutions in terms of graph Laplacian. Examples are Spectral graph CNN \cite{bruna2013spectral}, Chebyshev network (ChebyNet) \cite{defferrard2016convolutional}, graph convolutional network  (GCN)  \cite{kipf2016semi}, graph wavelet network \cite{xu2019graph}, GraphHeat network \cite{xu2019graphheat} etc.

Regularization in graphs is realized with the help of graph Laplacian. A smoothness functional on graphs can be obtained in terms of Laplacian and by processing on its eigenfunctions, regularization properties on graphs can be achieved. This has been utilized for inference in the case of semi-supervised classification, link prediction etc, \cite{zhou2004regularization}, \cite{belkin2004regularization}, and for graph kernel designs \cite{smola2003kernels}.


In this work, we propose a framework for filter designs in spectral GCNNs from which its regularization behavior can be analyzed.  We observed that by processing on the eigenvalues of the Laplacian, different kinds of filters with their corresponding regularization behavior can be obtained. We identified the condition on the spectrum of the Laplacian that ensures regularized spectral filter designs. The filters used in state-of-the-art networks can be deduced as special cases of our framework.  We also proposed a new set of filters inspired by the framework and identified certain state-of-the-art models as special cases.
The rest of the paper is organized as follows. Section \ref{sec:related works} discusses the related work. Section \ref{sec:notations} defines the notations used in the manuscript. Spectral GCNNs are briefly discussed in Section \ref{sec:gcns}. Section \ref{sec:regularization} discusses the framework for designing regularized graph convolution filters. Section \ref{sec:experiments} discusses the experiments and results. Conclusions are made in Section \ref{sec:conclusion}.

\section{Related works}\label{sec:related works}

\textbf{Spectral GCNNs:} The spectral filter designs for GCNNs start with the work by Bruna et.al \cite{bruna2013spectral} where the graph Fourier transform of the signals on nodes is utilized. The filter is then defined in terms of the eigenvectors of graph Laplacian. In ChebyNet \cite{defferrard2016convolutional},  convolution is done in the spectral domain. Convolution filter is then defined as the linear combination of powers of the Laplacian. The scaling factors of the linear combination are considered as the parameters to be learned from the network.  Kipf et.al \cite{kipf2016semi} proposed
graph convolutional network as the first-order
approximation to spectral convolutions defined in ChebyNet.  GraphHeat \cite{xu2019graphheat} uses negative exponential processing on the eigenvalues of Laplacian to improve the smoothness of the function to be learned by penalizing high-frequency components of the signal. Li et al \cite{li2019label} proposed improved graph convolution networks (ICGN) by proposing higher-order filters used in \cite{kipf2016semi}.

\textbf{Regularization in graphs:} Regularization associated with graphs has emerged along with the development of algorithms related to semi-supervised learning \cite{belkin2006manifold}.  The data points are converted into a network and the unknown label of a data point can be learned in a transductive setting. A generalized regularization theory is formulated where the classical formulation is augmented with a smoothness functional on graphs in terms of its Laplacian. These concepts are used for semi-supervised learning and in related applications  \cite{belkin2004regularization}, \cite{zhou2004regularization}, \cite{zhu2003semi}, \cite{weston2012deep}. Smola et.al \cite{smola2003kernels} leverages these concepts to define support vector kernels on graph nodes and they also formulated its connection with regularization operators. Our work is mainly motivated by this work.

\textbf{Spectral analysis of GCNNs:} The low pass filtering property of ChebyNet and GCN networks are analyzed by Wu et.al \cite{pmlr-v97-wu19e}. They have shown that the renormalization trick \cite{kipf2016semi} applied to GCN shrinks the spectrum of the modified Laplacian from $[0, 2]$ to $[0, 1.5]$ which favors the low pass filtering process. Li et.al \cite{li2019label} and Klicpera et.al \cite{klicpera2019diffusion} have given the frequency response analysis of their proposed filters. Compared to these works, we form a generalized framework for designing filters based on their regularization property. Adding to this, Gama et.al \cite{gama2019stability} studies the perturbation of graphs, consequent effects in filters, and proposed the conditions under which the filters are stable to small changes in the graph structure. 

\section{Notations }\label{sec:notations}

We define $G=(V,E)$ as an undirected graph, where $V$ is the set of $n$ nodes and $E$ is the set of edges. The adjacency matrix is defined as $W$ where $W_{ij} = w_{ij}$ denotes the weight associated with the edge $[i, j]$  and otherwise $0$. The degree matrix, $D$, is defined as the diagonal matrix where $D_{ii} = \sum_j w_{ij}$. The Laplacian of $G$ is defined as $L := D-W$ and the normalized Laplacian is defined as $\tilde{L} := D^{-\frac{1}{2}} L D^{-\frac{1}{2}} = I-D^{-\frac{1}{2}}WD^{-\frac{1}{2}}$. As $\tilde{L}$ is a real symmetric positive semi definite  matrix, it has a complete set of orthonormal eigenvectors  $\{ u_l\}_{l=1}^{n} \in \mathbb{R}^n$, known as the graph  \textit{Fourier modes} and the associated ordered real non negative eigenvalues  $\{ \lambda_l\}_{l=1}^{n}$, identified as the \textit{frequencies} of the graph. Let the eigen decomposition of $\tilde{L}$ be $U \Lambda U^T$ where $U$ is the matrix of eigenvectors $\{ u_l\}_{l=1}^{n}$ and $\Lambda$ is the diagonal matrix of eigenvalues. Graph Fourier Transform (GFT) of a signal $f:V\rightarrow \mathbb{R}$   is defined as $\hat{f} = U^Tf$ and inverse GFT is defined as $f = U\hat{f}$ \cite{shuman2013emerging}. 


\section{Spectral graph convolution networks}\label{sec:gcns}


Spectral convolutions on graphs can be defined as the multiplication of a signal on nodes with a graph filter. We define  $f = (f_1,\dots, f_n) \in \mathbb{R}^n$ as a signal on $n$ nodes on a graph $G$.  Graph filter, $\mathcal{F}$, can be regarded as a function that takes $f$ as input and outputs another function $y$. The convolution operation can be written as, $y= \mathcal{F}f = U g_\theta(\Lambda) U^T f$, 
where  $g_\theta(\Lambda) \in \mathbb{R}^{n\times n} = \mbox{diag}(g_\theta(\lambda_1),\dots,g_\theta(\lambda_n))$ is a diagonal matrix.  The function $g_\theta() : \mathbb{R} \rightarrow \mathbb{R}$ (with parameters $\{\theta\}$) is defined as \textit{frequency response} function of the filter $\mathcal{F}$.

The graph filters associated with state-of-the-art GCNNs correspond to a unique \textit{frequency response} function as listed in Table \ref{gcnnlist}. In spectral CNN \cite{bruna2013spectral}, output is defined as, $ y = \sigma (U^T \mbox{diag} (\alpha) Uf)$, 
where $\sigma$ is a non-linearity function and $\alpha = \kappa(\theta)$ is a matrix characterized by a spline kernel parameterized by $\theta$. The drawback of this model is that it is of $\mathcal{O}(n)$ and not localized in space. 
These problems are rectified by ChebyNet \cite{defferrard2016convolutional} where a polynomial filter is used. It helped to reduce the number of parameters to $k < n$ and the powers of Laplacians in the filter design solves the localization issue.

The graph convolutional network (GCN) \cite{kipf2016semi} is the first-order approximation of ChebyNet except for a sign change $(\theta_0 = -\theta_1 = \theta)$ and improved graph convolutional network (IGCN) \cite{li2019label} uses higher orders of GCN filters which can be inferred from the Table \ref{gcnnlist}.  The \textit{frequency response} of all the networks is parameterized by $\theta$. The network is optimized to learn these parameters with respect to the associated loss function for the downstream task. In the case of ChebyNet and IGCN, $K \in \mathbb{N}$ and for GraphHeat, $s \geq 0$ is a hyper-parameter.

In our work, the objective is to propose a framework  to design regularized  graph convolution filters and for this, we make use of regularization theory over graphs via graph Laplacian as discussed in the following section. We also found that the spectral filters in Table \ref{gcnnlist} can be deduced as special cases of the proposed \textit{frequency response} functions in our  framework.

\section{Regularized graph convolution filters}\label{sec:regularization}
We consider the signals on the nodes of the graph are generated from the function $f:V \rightarrow \mathbb{R}$. It is being identified that the eigenvectors of $L$ corresponding to lower frequencies or smaller eigenvalues are smoother on graphs \cite{zhu2009introduction}.  The smoothness corresponding to the $k-$th eigenvector is,
\begin{equation}\label{eqn:smoothness}
\sum_{i \sim j} w_{ij}[u_k(i)-u_k(j)]^2 = u_k^TLu_k = \lambda_k
\end{equation}
From Equation \ref{eqn:smoothness}, we can infer that a smoothly varying graph signal corresponds to eigenvectors with smaller eigenvalues. This  is under the assumption that the neighborhood of topologically identical nodes would be similar. In real-world applications, the signals over the graph could be noisy. In this context, we should filter out high-frequency content of the signal as it contains noise and low-frequency contents (eigenvectors corresponding to lower eigenvalues) should be maintained as it contains robust information. In other words, smoothness corresponds to spatial localization in the graphs which is important to infer local variability of the node neighborhoods. This is where the regularization behavior of the \textit{frequency response} functions of a GCNN becomes important. Extending Equation \ref{eqn:smoothness},   we can  define the smoothness functional on graph $G$ as, $S_G(f)  =  \sum_{i \sim j} w_{ij}  (f_i-f_j)^2 = f^TLf$. 

 \begin{table}
	\caption{Frequency response function and output of spectral filters of GCNNs}
	\label{gcnnlist}
	\centering
	\begin{tabular}{lll}
		\toprule
		Network     				 & Freq. response ($g_\theta(\lambda)$)			& Output, $y$   \\
		\midrule
		ChebyNet \cite{defferrard2016convolutional}	& $  \sum_{k = 0}^{K-1}\theta_k \lambda^k$	& $  U(\sum_{k = 0}^{K-1} \theta_k \Lambda^k )U^T f = (\theta_0 I + \sum_{k=1}^{K-1} \theta_k \tilde{L}^k)f$ \\
		GCN \cite{kipf2016semi}    	 & $ \big(\theta (1 - \lambda) \big)$ & $ \theta(I-\tilde{L})f$ \\
		GraphHeat \cite{xu2019graphheat}   		 & $  \theta_0 + \theta_1 \mbox{exp }(-s\lambda))$		&   $   (\theta_0I + \theta_1 e^{-s\tilde{L}}) f $     \\
		IGCN \cite{li2019label} & $   \big(\theta (1 - \lambda) \big)^K$ & $ \theta(I-\tilde{L})^Kf$\\
		\bottomrule
	\end{tabular}
\end{table}


The smoothness property associated with $L$ or $\tilde{L}$ also indicates its potential application to design regularized filters for GCNNs. Since the spectrum of $\tilde{L}$ is limited in $[0, 2]$, in this work we use normalized Laplacian. In the following section, we discuss how graph Laplacian can be used for the regularization in graphs and we propose our framework to design regularized graph convolution filters.

\subsection{Graph Laplacian and regularization}
Regularization functionals on $\mathbb{R}^n$ can be written as
\begin{equation}\label{lap_cont}
\langle f, Pf \rangle = \int |\bar{f}(\omega)|^2 r(\Vert \omega \Vert^2) d\omega = \langle f, r(\Delta)f \rangle
\end{equation}
where $f \in L_2(\mathbb{R}^n)$, $P$ is a regularization operator, $\bar{f}(\omega)$ denotes Fourier transform of $f$, $r(\Vert \omega \Vert^2)$ is a frequency penalizing function and $r(\Delta)$ is a function that acts on spectrum of the continuous Laplace  operator $\Delta$ \cite{smola2003kernels}. Equation \ref{lap_cont} can be mapped into the case of graphs by making an analogy between the continuous Laplace  operator and its discrete counterpart which is the graph Laplacian $\tilde{L}$.
Analogous to Equation \ref{lap_cont}, Smola et.al \cite{smola1998connection} used a function of Laplacian, $r(\tilde{L})$, in the place of $P$ in Equation \ref{lap_cont} under Laplacian's capability to impart a smoothness functional on graphs. Hence regularization functionals on graphs can be written as $\langle f, Pf \rangle = \langle f, r(\tilde{L})f \rangle$, 
where $r(\tilde{L})= \sum_{i=1}^n r(\tilde{\lambda_i}) u_i u_i^T$.

The choice of $r(\lambda)$ should be in such a way that it favors the \textit{low pass filtering} of the graph convolution filter, i.e, the function $r(\lambda)$ should be high for a higher value of $\lambda$ to impose more penalization on high frequency (high eigenvalue) content of the graph signal. Similarly, the penalization of low frequency should be less. Hence we name $r(\lambda)$ as \textit{regularization function}. The examples for choices of $r(\lambda)$ are listed in the second column of Table \ref{filters:table} and the functions are plotted in Figure \ref{filters:graphplot}.

%
%
%
 \begin{table}
	\caption{Filters, corresponding regularization function $(r(\lambda))$ and its filter definition}
	\label{filters:table}
	\centering
	\begin{tabular}{lll}
		\toprule
		Filter     				 & Regularization function $(r(\lambda))$ 			& Filter definition   \\
		\midrule
		Regularized Laplacian    & $1 + s\lambda, \; s > 0$  										&  $(I + s\tilde{L})^{-1}$ \\
		Diffusion				 & $exp \;\big(s\lambda\big), \; s > 0$  					& $exp \;(-s\tilde{L})$ \\
		$p$-step random walk    	 & $(aI - \lambda)^{-p} \; a \geq 2, \; p \in \mathbb{N}$								& $(aI - \tilde{L})^p$ \\
		Cosine    		 & $\big(cos \;\frac{\lambda\pi}{4}\big)^{-1}$		&   $cos\; (\frac{\tilde{L}\pi}{4})$     \\
		
		\bottomrule
	\end{tabular}
\end{table}

\begin{figure}[ht] 
	
	\begin{center}
		\centerline{\includegraphics[scale=0.345]{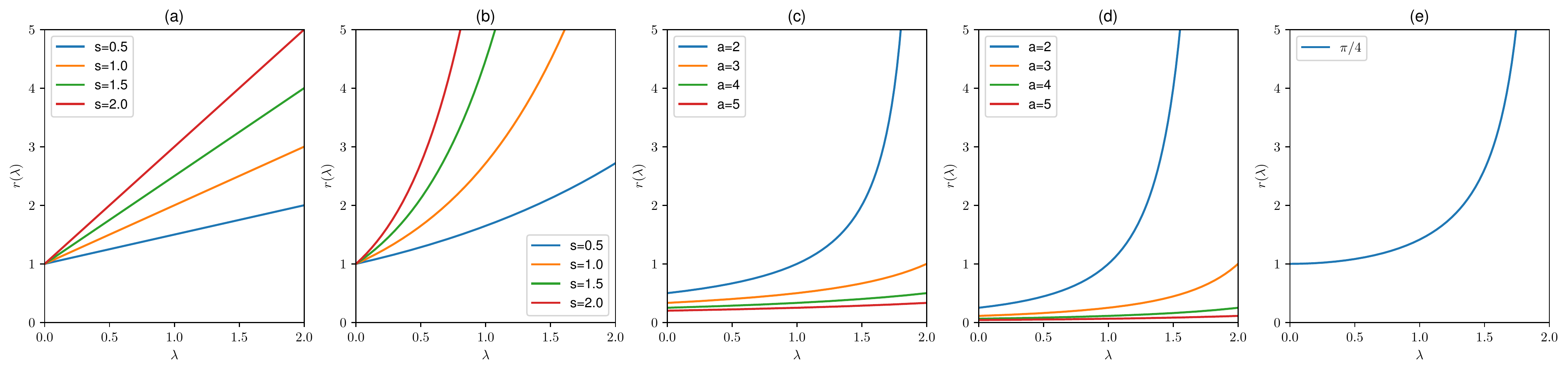}}
		\caption{Regularization function, $r(\lambda)$. (a) regularized Laplacian ($s = \{0.5, 1, 1.5,2\}$),
			(b) diffusion function ($s = \{0.5,1,1.5,2\}$), (c) one-step random walk ($a = \{2,3,4,5\}$), (d) 2-step random walk ($a = \{2,3,4,5\}$), (e) inverse cosine function.}
		\label{filters:graphplot}
	\end{center}
	
\end{figure}

\begin{remark} \label{remark1}
	There exists an inverse relationship between the \textit{regularization function}  and \textit{frequency response} function. To impose high penalization on higher frequencies, \textit{regularization function} is supposed to be a monotonically increasing function of the eigenvalues. At the same time, for the low pass filtering characteristics, to make high filter gain on a lower frequency and vice versa, the \textit{frequency response} function is supposed to be a monotonically decreasing function of the eigenvalues.
\end{remark}

\begin{remark}\label{remark2}
	Smola et.al \cite{smola1998connection} has shown that $P^{-1}$ (psuedo-inverse if not invertible) is a positive semidefinite (p.s.d) support vector kernel in a reproducing kernel Hilbert space (RKHS) $\mathcal{H}$ where $P \in \mathbb{R}^{n \times n}$ is a positive semidefinite regularization matrix and $\mathcal{H}$ is the image of $\mathbb{R}^n$ under $P$. 
\end{remark}

Remark \ref{remark1} and \ref{remark2} points out in the direction of using the inverse of a \textit{regularization function}, via $\big(r(\tilde{L})\big)^{-1}$, as a \textit{frequency response} function for spectral filters. In this context, regularized filters corresponding to their \textit{regularization functions} can be obtained as, 
\begin{equation}\label{kernel}
\mathcal{F}=\big(r(\tilde{L})\big)^{-1} = \sum_{i=1}^n \big(r(\lambda)\big)^{-1}u_iu_i^T
\end{equation}
where $\{(u_i,\lambda_i)\}$ is the eigensystem of $\tilde{L}$ and $\big(r(\lambda)\big)^{-1}$ is the reciprocal function of \textit{regularization function}. In the context of Remark \ref{remark2}, we can see that filters of GCNNs defined as per Equation \ref{kernel} are also support vector kernels on graphs provided their parameterization (if any) maintains positive semidefiniteness. The detailed discussion is provided in Appendix \ref{app1}.




\subsection{A framework for designing  graph convolution filters}\label{sec:framework}

Regularized graph filters  which are defined as follows can be designed by making use of Equation \ref{kernel}.
\begin{definition}[Regularized graph convolution filter]: The graph filter whose \textit{frequency response} function  $g_\theta(\lambda)$ behaves like a low-pass filter, i.e,  $g_\theta(\lambda)$ should be a monotonically decreasing function in $\lambda$ or equivalently the associated regularization function $r(\lambda) = \big(g_\theta(\lambda)\big)^{-1}$ should be a monotonically increasing function in $\lambda$. 
\end{definition} 

The design strategy for regularized graph convolution filters is summarized in Theorem \ref{th1}.
\begin{theorem}\label{th1}
	
A monotonically increasing function in the interval $[0, \lambda_{max}]$ is a valid regularization function to design regularized graph convolution filters using \eqref{kernel} where $\lambda_{max}$ is the maximum eigenvalue of $\tilde{L}$.
\end{theorem}

\begin{proof}
	Note that Laplacian $\tilde{L}$ can be decomposed as $U \Lambda U^T$. 
	Equivalently, this decomposition can be considered as a sum of  matrix of projections onto  one-dimensional subspace spanned by the eigenvectors (Fourier basis), i.e, $\tilde{L} = \sum_{i=1}^n \lambda_i P_{\lambda_i}$, where the linear map $P_{\lambda_i}(x) =  u_iu_i^Tx$ is the orthogonal
	projection onto the subspace spanned by the Fourier basis vector $u_i$. 
	Consider a regularization function $r(\lambda)$ that is monotonically increasing.  
	Note that frequency response function $g_\theta(\lambda) = 1/ r(\lambda)$ is monotonically decreasing. The filtering operation of a signal $f$ by the filter $\mathcal{F}$ can be written in terms of the mappings $P_{\lambda_i}, 1 \leq i \leq n$, i.e,  $$y = \mathcal{F}f = \sum_{i=1}^n g(\lambda_i) P_{\lambda_i}(f) = \sum_{i=1}^n g_\theta(\lambda_i) u_iu_i^T f$$ where $\{(\lambda_i, u_i)\}$ constitutes the eigensystem of $\tilde{L}$. 
	The values $g_\theta(\lambda_i)$ can be considered as weights that measure the importance of the corresponding eigenspace in the  amplification or attenuation of the signal $f$. As $g(\lambda)$ is monotonically decreasing the weight of eigenspace corresponding to eigenvectors of lower frequencies (lower eigenvalues) of $\tilde{L}$ are higher and vice versa. The filter gain of the lower frequency components of $f$ are higher compared to the higher frequencies or $\mathcal{F}$ is a low pass filter. The validity of $r(\lambda)$ is established by the low pass filtering and hence the proof.  
\end{proof} 

Note that the  theorem also holds for other definitions of normalized or unnormalized Laplacian and any spectrum in $[0, \infty]$, provided the monotonicity property is maintained. In the context of Theorem \ref{th1}, to design regularized filters, it is enough to pick a \textit{regularization function} $r(\lambda)$ with the monotonicity property and to plug into Equation \ref{kernel} to define the filter. 

\subsubsection{Factors affecting the choices of the regularization function} \label{sec:reg_discuss}
We can have custom designs for the regularization function. However,
to define a closed-form expression for the filter $\mathcal{F}$, the \textit{regularization function} $r(\lambda)$ should be able to be expressed in a closed-form. Hence the choice of $r(\lambda)$ should be limited to the functions with power series expansion to get closed-form expressions for easier computations. 

 The powers of the Laplacian involved in the expression can also affect graph learning since $(L^K)_{ij} = 0$ if the shortest path distance between nodes $i$ and $j$ is greater than $K$ \cite{hammond2011wavelets}. For example, the regularization function form of one step random walk (where $a$ = 2) and inverse cosine is approximately same. But the computation of cosine filter involves higher-order even powers of Laplacian whose non zero elements are determined by graph structure. Similarly, by knowing the spectrum of the Laplacian, it is possible to precompute the values of hyper-parameters to precisely design the form of the regularization function that spans in the spectrum. In the next section, we discuss a set of filters corresponding to the regularization functions in Table \ref{filters:table}.

\subsection{Regularized filters for GCNNs}

We take Equations in the second column of the Table \ref{filters:table}and plug into Equation \ref{kernel} to define the regularized filters. The results are summarized in the third column of Table \ref{filters:table}. Note that variants of some filters are already familiar in the literature as explained below.

\textbf{Case 1:} In a $p$-step random walk filter if we put $a=1$ and $p=1$, we get the filter corresponding to GCN \cite{kipf2016semi}. \textbf{Case 2:}  The filter used in IGCN \cite{li2019label} uses higher powers of the GCN filter. Hence it corresponds to a $p$-step random walk filter with a value of $p \geq 2$ and $a=1$. \textbf{Case 3:} As per  \cite{li2019label}, the graph filter of the label propagation (LP) method for semi-supervised learning takes the form of the regularized Laplacian filter. \textbf{Case 4:} GraphHeat filter is similar to a \textit{diffusion} filter together with an identity matrix.

\textbf{Computational complexity: }For \textit{regularized Laplacian}, learning complexity costs $\mathcal{O}(n^3)$ as it involves matrix inversion. For other filters, it is  $\mathcal{O}(K|E|)$, where $K$ is the maximum power of  the Laplacian involved in the approximation of the equation of the filter.


\subsection{Analysis of regularization behavior of GCNNs} \label{sec:reg_behave}

The idea is to  identify the \textit{regularization function} $r(\lambda)$ corresponding to the state-of-the-art networks from their filter definition. This helps to analyze the regularization capability of their filters. 

\textbf{Chebynet: }ChebyNet filtering \cite{defferrard2016convolutional} is defined as, $
y  = U(\sum_{i = 0}^{K-1} \theta_k \Lambda^k )U^T f \approx \mbox{exp} (\tilde{L})f$, where we assume parameters $\theta_k$ are the coefficients of the expansion of matrix exponential $\mbox{exp} (\tilde{L})$. The regularization function $r(\lambda) =  \big(g_\theta(\lambda)\big)^{-1} = (\sum_{i = 0}^{K-1}\theta_k \lambda^k)^{-1} \approx c.\mbox{exp }(-\lambda)$ where $c$ is a constant determined by $\theta_k$'s, the parameters learned by the network. Hence the regularization happening in ChebyNet is the exact opposite of the expected behavior since small eigenvalues are attenuated more and large ones are attenuated less as shown in Figure \ref{graph:gcnns}(a). To improve the regularization property of ChebyNet, the filtering shall be done as per the negative exponential or\textit{ diffusion regularization} function and for this, we shall change the filtering operation as, $y  = U(\sum_{i = 0}^{K-1} (-1)^k \theta_k \Lambda^k )U^T f \approx \mbox{exp} (-\tilde{L})f  $
with the constraint $\theta_k > 0$ to keep up with the desired regularization property. We can also bring an additional hyper-parameter $s > 0$ into the power of exponential function. Note that in this case, it corresponds to the \textit{diffusion} filter. 

\begin{figure*}[ht]
 
	\begin{center}
		\centerline{\includegraphics[scale=0.35]{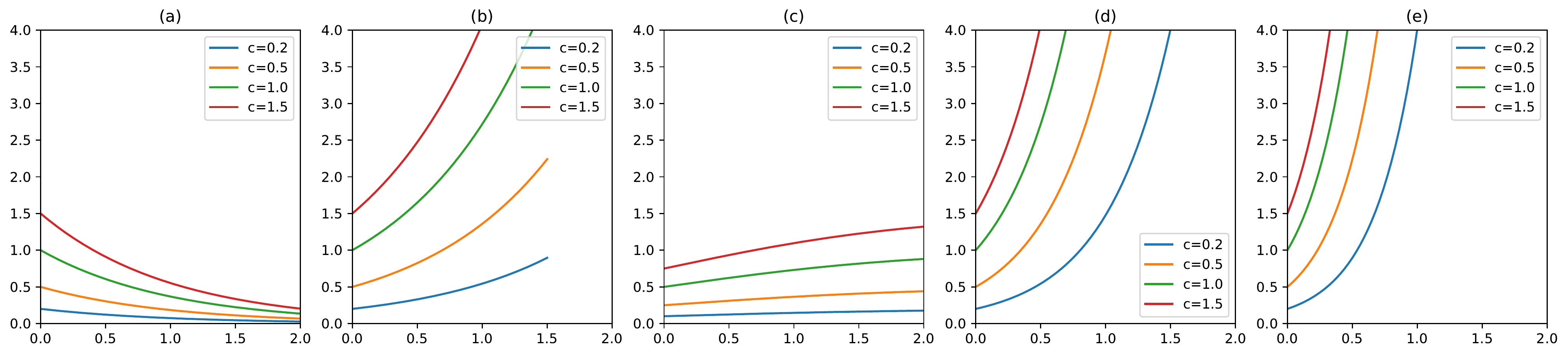}}
		\caption{Regularization function, $r(\lambda)$. (a) ChebyNet,
			(b) GCN, (c) GraphHeat, (d) IGCN for $k=2$ , (e) IGCN for $k=3$. All graphs are for ($c = \{0.2, 0.5, 1.0, 1.5\}$) }
		\label{graph:gcnns} 
	\end{center}

\end{figure*}

\textbf{GCN: } GCN filtering \cite{kipf2016semi} operation can be written as, $y  = \theta(I-\tilde{L})f  \approx \mbox{exp}(-\lambda)$,
where we assume parameter $\theta$ is 1 in the exponential approximation. The regularization function $r(\lambda) = c. (1 - \lambda)^{-1} \approx c. \mbox{exp}(\lambda) $ where  $c$ is a constant determined by the parameter $\theta$. Hence the regularization happening in GCN is as desired as shown in Figure \ref{graph:gcnns}(b).
Note that the filter of GCN corresponds to the first-order approximation of the \textit{diffusion} filter. So in effect, it is the diffusion process that harnesses the representation capability of GCN by changing the sign of parameters $(\theta_0 = -\theta_1 = \theta)$ compared to ChebyNet as explained in Section \ref{sec:gcns}.

\textit{Spectral analysis of renormalization trick  :} The trick refers to the process of adding self-loops \cite{kipf2016semi} to the graphs for stable training of the network. Wu et.al \cite{pmlr-v97-wu19e} have shown that adding self-loops helps to shrink the Laplacian spectrum form [0, 2] to [0, 1.5] which boosts the low pass filtering behavior. So the \textit{regularization function} remains in the same form as mentioned above, but the range of eigenvalues being in the interval [0, 1.5].


\textbf{GraphHeat: }GraphHeat filtering \cite{xu2019graphheat} operation can be written as,
$y = (\theta_0I + \theta_1 e^{-s\tilde{L}}) f$. The \textit{regularization function} $r(\lambda) = c.(1 +   \mbox{exp }(-s\lambda))^{-1} $ where we assume $c$ is a factor determined by $\theta_0$ and $\theta_1$. The regularization function is shown in Figure \ref{graph:gcnns}(c). 

\textbf{IGCN: }Improved graph convolutional network (IGCN) filtering \cite{li2019label} operation can be written as,
$y  = \theta(I-\tilde{L})^kf  \approx \mbox{exp}(-k\lambda)$.
The regularization function $r(\lambda) = c. (1 - \lambda)^{-k} \approx c. \mbox{exp}(k\lambda) $ where   $c$ is a constant determined by the parameter $\theta$. Hence the regularization happening in GCN is as desired as shown in Figure \ref{graph:gcnns} (d).

\section{Experiments}\label{sec:experiments}

The variants of proposed filters as in Table \ref{filters:table} is compared with state-of-the-art GCNNs namely ChebyNet \cite{defferrard2016convolutional}, GCN \cite{kipf2016semi}, GraphHeat \cite{xu2019graphheat}, and IGCN (RNM variant of the filter)  \cite{li2019label}. The comparison is also made with graph regularization based algorithms for semi-supervised learning namely - manifold regularization (ManiReg) \cite{belkin2006manifold}, semi-supervised embedding (SemiEmb) \cite{weston2012deep}, and label propagation (LP) \cite{zhu2003semi}. Other baselines used are Planetoid \cite{yang2016revisiting}, DeepWalk \cite{perozzi2014deepwalk}, and iterative classification algorithm  (ICA) \cite{lu2003link}. Citation network datasets \cite{yang2016revisiting} - Cora, Citeseer,  and Pubmed  are used for the study. In these graphs, nodes represent documents, and edges represent citations. The datasets also contain 'bag-of-words' feature vectors for each document. 

\begin{table}[t]
	\caption{Classification accuracy (in percentage $\pm$ standard deviation) along with average time taken for one epoch (in brackets) and parameters as triplets in the order of datasets ($f$ - number of filters).}
	\label{results}
	\centering
	\scalebox{0.925}{
		\begin{tabular}{llllll}
			\toprule
			Methods     	&	Cora		 & Citeseer		& Pubmed    \\
			\midrule
			ManiReg  & 59.5 & 60.1 & 70.7  \\
			SemiEmb  & 59.0 & 59.6 & 71.1  \\
			LP       & 68.0 & 45.3 & 63.0 \\
			DeepWalk & 67.2 & 43.2 & 65.3 \\
			ICA      & 75.1 & 69.1 & 73.9 \\
			PLanetoid& 75.7 & 64.7 & 77.2  \\
			MLP		 & 56.2 & 57.1 & 70.7  \\
			\bottomrule

			GCN  & 81.78 $\pm$ 0.64 (1.02)  & 70.73 $\pm$ 0.53 (1.03) & 78.48 $\pm$ 0.58 (1.21)   \\
			IGCN     &  80.49 $\pm$ 1.58 (1.02) & 68.86 $\pm$ 1.01 (1.06)  & 77.87 $\pm$ 1.55 (1.25)  \\
			ChebyNet  &  82.16 $\pm$ 0.74 (1.03) & 70.46 $\pm$ 0.70 (1.04) &  78.24  $\pm$ 0.43 (1.21)  \\
		 
			GraphHeat       & 81.38 $\pm$ 0.69 (1.04) & 69.90 $\pm$ 0.50 (1.05) &  75.64 $\pm$ 0.64 (1.34)   \\
		 
			\bottomrule
			Diffusion  & \textbf{83.12 $\pm$ 0.37} (1.11)  & 71.17 $\pm$ 0.43 (1.06)  &  \textbf{79.20 $\pm$ 0.36} (1.80)  \\
		 
			1-step RW  &  82.36 $\pm$ 0.34 (1.02) &  71.05 $\pm$ 0.34 (1.03) &  78.74 $\pm$  0.27 (1.21)  \\
		 
			2-step RW       & 82.51 $\pm$ 0.22 (1.03)  & 71.18 $\pm$ 0.59 (1.05)  &  78.64 $\pm$ 0.20 (1.29)   \\
		 
			3-step RW   &  82.56 $\pm$ 0.24 (1.05) & \textbf{71.21 $\pm$ 0.63} (1.04)  &  78.28 $\pm$ 0.36 (1.81)  \\
		 
			Cosine      &  75.53 $\pm$ 0.52 (1.03)  &  67.29 $\pm$ 0.64 (1.03) &  75.52 $\pm$ 0.53 (1.29)   \\
		 
			
			\bottomrule	
		\end{tabular}
	}
\end{table} 

\subsection{Experimental setup}
For GCNN models, network architecture proposed by Kipf et.al \cite{kipf2016semi} is used for the experiments. Networks with one layer, two layers, and three layers of graph convolution (GC) are used to evaluate all the filters under study. Along with this, networks with a GC layer followed by one and two layers of dense layers are also studied. In the experiments, it has been found that the network with two layers of GC has outperformed other architectures. It
that takes the form $$Z = \mbox{softmax}(\mathcal{F}(\tilde{L}) \; \mbox{ReLU}(\mathcal{F}(\tilde{L})X\Theta^{(1)}) \Theta^{(2)} )$$ where $\mathcal{F}(\tilde{L}) \in \mathbb{R}^{n \times n}$ is the filter, $X \in \mathbb{R}^{n \times d}$ is the input feature matrix, $\theta^{(1)} \in \mathbb{R}^{d \times c_1}$ is the filter parameters of first layer ($c_1$ is the number of filters) and $\theta^{(2)} \in \mathbb{R}^{c_1 \times c_2}$  is the filter parameters of second layer ($c_2$ is the number of filters). Note that the value of $c_2$ equals the total number of classes in the data output. The loss function optimized is the cross-entropy error over the labeled examples \cite{kipf2016semi} defined as follows.
$$\mathcal{L} = - \sum_{i \in \mathcal{Y}} \sum_{j=1}^{c_2} y_{ij} \;\mbox{ln} (Z_{ij})$$
where $\mathcal{Y}$ is the set of nodes whose labels are known and $y_{ij}$ is defined as 1 if label of node $i$ is $j$ and 0 otherwise.
For training, all the feature vectors and 20 labels per class are used. The same dataset split as used by Yang et.al  \cite{yang2016revisiting} is followed in the experiments. All the GCNN models corresponding to different filters are trained 10 times each according to a unique random seed selected at random. 
All models are trained for a maximum of 200 epochs using the ADAM optimizer \cite{kingma2014adam} with the learning rate fixed as 0.01. Early stopping is done in the training if the validation loss does not decrease for 10 consecutive epochs. Network weight initialization and normalization of input feature vectors of the nodes are done as per \cite{glorot2010understanding}.  Implementation is done using Tensorflow \cite{abadi2016tensorflow}. The hardware used for the experiments is Intel Xeon E5-2630 v3 2.4 GHz CPU, 80 GB RAM, and Nvidia GeForce GTX 1080-Ti GPU. Accuracy is used as the performance measure where models are evaluated on a test set of 1000 labeled examples. Since the focus of the study is on comparison of spectral filters, for GCNN variants mean accuracy along with standard deviation is reported. For algorithms other than GCNN, accuracy on a single dataset split reported by Kipf et.al \cite{kipf2016semi} is given, since we also follow the same dataset split. 

The higher powers of graph Laplacian is computed with the Chebyshev polynomial approximations \cite{hammond2011wavelets} for ChebyNet and $p$-step random walk filters considering its computational advantage.

Chebyshev polynomial of order $k$ is computed by the recurrence relation $$ T_k(x) = 2xT_{k-1}(x) - T_{k-2}(x),$$ where $T_0$ and $T_1$ is defined as 1 and $x$ respectively.
The polynomials form an orthogonal basis for $L^2([-1,1], \frac{dy}{\sqrt{1-y^2}})$, i.e, the space of square integrable functions with respect to the measure $dy/\sqrt{1-y^2}$. Hence when it comes to computing the powers of graph Laplacian ($L$), its spectrum has to be rescaled in the interval [-1, 1] as follows $$L_s = \frac{2}{\lambda_{max}}L - I_n,$$ where $L_s$ is the rescaled Laplacian, $\lambda_{max}$ is the maximum eigenvalue of $L$ and $n$ is the number of nodes in the graph.

But for GraphHeat, diffusion and cosine filters direct computation is done. As these filters involve a hyper-parameter being multiplied to the Laplacian matrix (unlike ChebyNet, and $p-$ step RW), Chebyshev approximation is not possible because when the rescaling of the matrix happens, the effect of this hyper-parameter multiplication is nullified.


\subsection{Results}\label{subsec:results}

The results are tabulated in Table \ref{results}. The best results are bolded. 
 The hyper-parameters used for the models are: dropout rate = 0.8, L2 regularization factor for the first layer weights = $5 \times 10^{-4}$, and the number of filters used in each layer is tuned from 16, 32, 64, and 128. These hyper-parameters are optimized on an additional validation set of 500 labeled examples as followed in \cite{kipf2016semi}. \textit{Diffusion filter} which is a matrix exponential is approximated for first $K+1$ terms, i.e, 
$g_\theta(\Lambda) = \theta \sum_{k = 0}^K (-1)^k \frac{1}{k!} \Lambda^k $ where there is only a single parameter $\theta$ is learned. 
For \textit{diffusion} filter, ChebyNet and GraphHeat  the value of $K$ used for the approximation of matrix exponential is tuned from $\{1, 2, 3, 4\}$. Similarly,  \textit{cosine filter} which involves cosine of a matrix is taken as $g_\theta(\Lambda) = \theta \sum_{k = 0}^K (-1)^k \frac{1}{2k!}  \Lambda^{2k}$. The value of $K$ is tuned from $\{1, 2, 3\}$. 
For \textit{diffusion} and GraphHeat filter, the value of $s$ is tuned in the range [0.5, 1.5] and for \textit{$p$-step randomwalk} filters, the value of $a$ is tuned in the range [2,24]. 
For GCN and IGCN, the author's code was reproduced for experiments.

\subsection{Discussion}\label{subsec:discussion}

Compared with graph regularization and label propagation methods, GCNN methods have better performance. The \textit{diffusion} filter has the highest accuracy in Cora and Pubmed. In Citeseer it is \textit{3-step RW} but the difference with \textit{diffusion} filter is negligible. Among methods other than GCNNs, ICA and Planetoid have better performance.
\textit{1-step RW} filter is an improvised version of GCN and IGCN in terms of regularization ability. Analyzing their comparison, \textit{1-step RW} filter performs better in all datasets. The 3 variants of RW filters have higher performance in Cora and Citeseer datasets compared with GCN and IGCN the reason being attributed to the tuning of parameter $a$ and in the case of Pubmed, RW filters and GCN have better performance than IGCN. 

\begin{figure}[ht]
 
	\begin{center}
		\centerline{\includegraphics[scale=0.35]{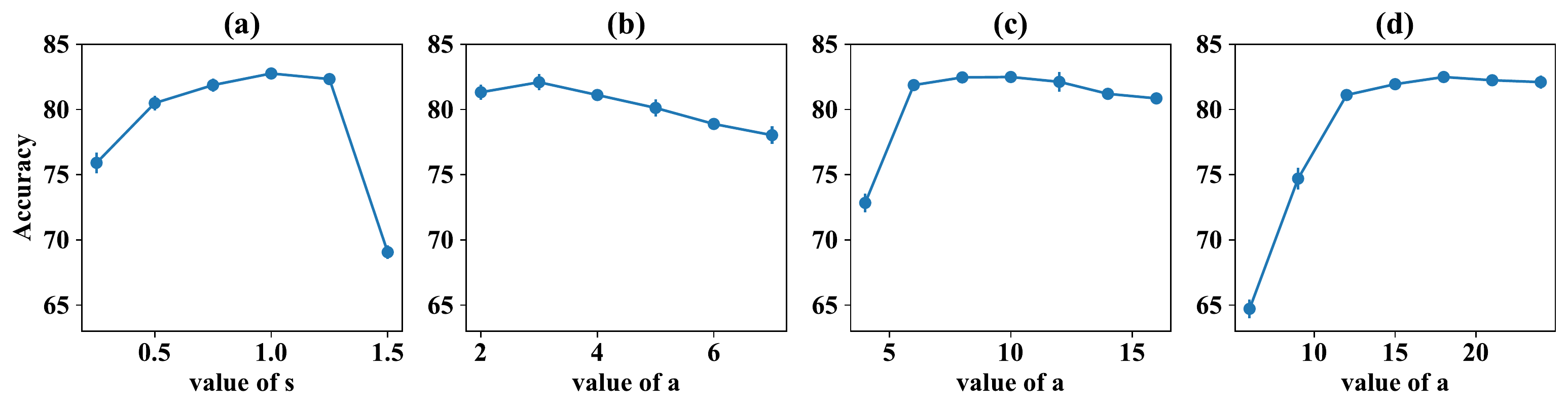}}
		\caption{Accuracy variation with hyper-parameters. (a) Diffusion,
			(b) 1-step RW, (c) 2-step RW, (d) 3-step RW   }
		\label{hyp:graph} 
	\end{center}
 
\end{figure}

ChebyNet, GraphHeat, and \textit{diffusion} filters calculate the first few powers of the Laplacian in their learning settings. It has been observed that the performance of the \textit{diffusion} filter is better than both despite using one parameter to learn. 
The performance of the \textit{cosine} filter is lower compared with other filters. The reason is due to the approximation of the \textit{cosine filter} that requires higher even powers of the Laplacian whose elements can be mostly zeros based on the graph structure as discussed in Section \ref{sec:reg_discuss}. It also requires skipping odd hopes in the graph that results in some information loss while learning.

The average time required for one training epoch is shown in the brackets along with the accuracy. It can be noted that the time taken increases as the higher order powers of the Laplacian and number of filters increase. 

\textbf{Effects of hyper-parameter tuning: }The variation in accuracy against hyper-parameters of the proposed filters applied to the Cora dataset is given in Figure \ref{hyp:graph}. For \textit{diffusion} filter ($K=3$), accuracy increases as $s$ is increased but there is a drop in accuracy after a peak value of $s$. Similar is the case of \textit{1-step RW}. In the case of \textit{2-step} and \textit{3-step RW} filters, accuracy increases as the value of $a$ increases but after a threshold point, accuracy variation is minimal. Similar is the trend observed for other datasets. For the experiments, the network with two layers of GC having 32 number of filters is used. 



\subsection{Decoupling low pass filtering from network learning}

To underline the practical impact of the framework we proposed, an experiment is done that decouples the low pass filtering from the network learning inspired by \cite{pmlr-v97-wu19e}.
 First, the filtering is done separately using $\mathcal{F}$ (with no learning parameters) and the resulting filtered features are given into a two-layer MLP (chosen after ablation studies among 1 \& 3 layer models). This helps to identify the impact of the choice of $r(\lambda)$ formulated in our work independent of the network parameters. The results are given in Table \ref{table}. 
 
  \begin{table}
 	
 	\caption{Accuracy of the filters.  Std dev. is given in brackets.}
 	\centering 	%
 	
 	\scalebox{1}{
 		\begin{tabular}{lllll}
 			\hline
 			Methods & $g(\lambda)$  & Cora  & Citeseer & Pubmed \\
 			\hline
 			MLP & -- & 56.50 (1.21) & 53.57 (2.71) &	71.87 (0.21)
 			\\

 			GCN & $(1-\lambda)$ & 77.99 (0.75) & 68.28 (0.51) & 76.05 (0.33)	 
 			\\
 			IGCN & $(1-\lambda)^K$& 81.44 (0.41) & 70.64 (0.67) & 78.46 (0.70)  
 			\\
 			ChebyNet & $\sum_{k=0}^{K-1} \lambda^k$ & 27.18 (1.46) & 28.63 (1.08) & 59.99 (0.58)
 			\\

 			GraphHeat & 1+ $\mbox{exp}(-s\lambda)$ & 73.57 (0.72) & 66.31 (0.59) & 73.50 (0.85)	  
 			\\
 			\hline 
 			Diffusion & $\mbox{exp}(-s\lambda)$ & 78.47 (0.32) & 68.47 (0.61) &	76.72 (0.59) 
 			\\
 			
 			1-step RW & $(a - \lambda)$ & 77.31 (0.52) & 67.95 (0.80) &	76.34 (0.47) 
 			\\
 			
 			2-step RW & $(a - \lambda)^2$ & 78.08 (0.62) & 69.41 (0.44) & 76.90 (0.23)
 			\\
 			
 			3-step RW & $(a - \lambda)^3$ & 78.98 (0.28) & 69.23  (0.71) & 71.49 (0.81)
 			\\
 			
 			Cosine & $\mbox{cos}(\lambda\pi/4)$& 70.38 (0.82)& 64.67 (0.62) & 72.01 (0.74)	 
 			\\
 			
 			\hline
 			
 		\end{tabular}
 		
 	}
 	\label{table}
 \end{table}

 \textbf{Observations:} In section \ref{sec:reg_behave}, we found that unlike other networks, $r(\lambda)$ of ChebyNet is the opposite of the required monotone property. This is evident from the results as its performance is lower than the rest including MLP. All other filters except the ChebyNet satisfy Theory \ref{th1}, and hence their performance is better. They also perform better than MLP which indicates the importance of the proposed framework. The results are lower compared with GCN architecture \cite{kipf2016semi} followed in the previous experiment. This points out to the possibility that the stochastic nature of neural networks may not guarantee the desired filtering properties. The case of the ChebyNet is an example as its performance is good in GCN architecture despite having contradictions with Theory \ref{th1} whereas its performance in the new experiment is lower.

\section{Conclusion}\label{sec:conclusion}

We formulated a framework to design regularized filters for GCNNs based on regularization in graphs modeled by graph Laplacian. A new set of regularized filters are proposed and identified the state-of-the-art filter designs as their special cases. The new filters designs proposed in the context of the framework has shown superior performance in semi-supervised classification task compared to conventional methods and state-of-the-art GCNNs. Considering the practical impacts of the framework we proposed, we also observed that the stochastic nature of the neural networks can possibly does not guarantee the desired low pass filtering property that has to be satisfied by the spectral GCNN filters. 

\section*{Acknowledgments}

\textit{The authors thank Subrahamanian Moosath K.S, IIST, Thiruvananthapuram, and Shiju S.S, Technical Architect, IBS Software Pvt. Ltd., Thiruvananthapuram for the productive discussions.}

 \appendix

\section{Regularization in graphs, support vector kernels and spectral GCNN filters}\label{app1}
\label{intro}

The support vector kernel $k: X \times X \rightarrow \mathbb{R}$ is considered as a similarity measure between a pair of data points in a space $X$. Support vector kernels can be formulated by solving the self-consistency condition (\cite{smola1998connection}),
\begin{equation} \label{self}
\langle k(x,.), Pk(x',.) \rangle = k(x,x')
\end{equation}
where $P$ is the regularization operator. 

From equation \ref{self}, Smola et.al  \cite{smola1998connection} found that given a regularization operator $P$, there exist a support vector kernel $k$ that minimize the regularized risk functional,

\begin{equation}\label{risk}
R_{reg}[f] = R_{emp} + \frac{\lambda}{2} \Vert Pf\Vert^2 
\end{equation}

that also enforce flatness (determined by $P$) in the feature space or Reproducing Kernel Hilbert Space (RKHS) of functions. They also found that given a support vector kernel $k$, regularization operator $P$ can be found out such that a regularization network \cite{smola1998connection} using $P$ is equivalent to a support vector machine that uses the kernel $k$. Note that $R_{emp}$ is the empirical loss function and $\lambda$ is a hyper-parameter. 

Smola et.al \cite{smola2003kernels} used the above  concepts to design support vector kernels on graphs. 
As shown in Equation \ref{kernel}, graph Laplacian ($\tilde{L}$) can be used to define a smoothness functional on graphs that aids in designing regularization operators. They proved that \textit{if $H$ is the image of $\mathbb{R}^n$ under $P \in \mathbb{R}^{n \times n}$ (a positive semidefinite regularization matrix), then $H$ whose dot product is defined as $\langle f, Pf \rangle $ is a  RKHS and the corresponding support vector kernel is defined as $k(i,j) = [P^{-1}]_{ij}$, where $P^{-1}$ denotes the psuedo-inverse if $P$ is not invertible.}

Now if we consider the case of GCNN filters, we can observe that if the parameters $\theta$s associated with the filter definition maintains positive definiteness of the matrix $\mathcal{F} = \big(r(\tilde{L})\big)^{-1}$, then the filter can be considered as valid and equivalent support vector kernel that solves the regularized risk functional in Equation \ref{risk}. The corresponding regularization behavior induced by $P$ in equation \ref{risk} can be attributed to the corresponding \textit{regularization function $r(\lambda)$}.
 

\begin{thebibliography}{10}
	
	\bibitem{726791}
	Y.~{Lecun}, L.~{Bottou}, Y.~{Bengio}, and P.~{Haffner}, ``Gradient-based
	learning applied to document recognition,'' {\em Proceedings of the IEEE},
	vol.~86, pp.~2278--2324, Nov 1998.
	
	\bibitem{8100059}
	F.~{Monti}, D.~{Boscaini}, J.~{Masci}, E.~{Rodolà}, J.~{Svoboda}, and M.~M.
	{Bronstein}, ``Geometric deep learning on graphs and manifolds using mixture
	model cnns,'' in {\em 2017 IEEE Conference on Computer Vision and Pattern
		Recognition (CVPR)}, pp.~5425--5434, July 2017.
	
	\bibitem{gilmer2017neural}
	J.~Gilmer, S.~S. Schoenholz, P.~F. Riley, O.~Vinyals, and G.~E. Dahl, ``Neural
	message passing for quantum chemistry,'' in {\em Proceedings of the 34th
		International Conference on Machine Learning-Volume 70}, pp.~1263--1272,
	JMLR. org, 2017.
	
	\bibitem{duvenaud2015convolutional}
	D.~K. Duvenaud, D.~Maclaurin, J.~Iparraguirre, R.~Bombarell, T.~Hirzel,
	A.~Aspuru-Guzik, and R.~P. Adams, ``Convolutional networks on graphs for
	learning molecular fingerprints,'' in {\em Advances in neural information
		processing systems}, pp.~2224--2232, 2015.
	
	\bibitem{kearnes2016molecular}
	S.~Kearnes, K.~McCloskey, M.~Berndl, V.~Pande, and P.~Riley, ``Molecular graph
	convolutions: moving beyond fingerprints,'' {\em Journal of computer-aided
		molecular design}, vol.~30, no.~8, pp.~595--608, 2016.
	
	\bibitem{hamilton2017inductive}
	W.~Hamilton, Z.~Ying, and J.~Leskovec, ``Inductive representation learning on
	large graphs,'' in {\em Advances in neural information processing systems},
	pp.~1024--1034, 2017.
	
	\bibitem{li2015gated}
	Y.~Li, D.~Tarlow, M.~Brockschmidt, and R.~Zemel, ``Gated graph sequence neural
	networks,'' {\em arXiv preprint arXiv:1511.05493}, 2015.
	
	\bibitem{velivckovic2017graph}
	P.~Veli{\v{c}}kovi{\'c}, G.~Cucurull, A.~Casanova, A.~Romero, P.~Lio, and
	Y.~Bengio, ``Graph attention networks,'' {\em arXiv preprint
		arXiv:1710.10903}, 2017.
	
	\bibitem{chung1997spectral}
	F.~R. Chung, {\em Spectral graph theory}.
	\newblock No.~92, American Mathematical Soc., 1997.
	
	\bibitem{shuman2013emerging}
	D.~I. Shuman, S.~K. Narang, P.~Frossard, A.~Ortega, and P.~Vandergheynst, ``The
	emerging field of signal processing on graphs: Extending high-dimensional
	data analysis to networks and other irregular domains,'' {\em IEEE signal
		processing magazine}, vol.~30, no.~3, pp.~83--98, 2013.
	
	\bibitem{bruna2013spectral}
	J.~Bruna, W.~Zaremba, A.~Szlam, and Y.~LeCun, ``Spectral networks and locally
	connected networks on graphs,'' {\em arXiv preprint arXiv:1312.6203}, 2013.
	
	\bibitem{defferrard2016convolutional}
	M.~Defferrard, X.~Bresson, and P.~Vandergheynst, ``Convolutional neural
	networks on graphs with fast localized spectral filtering,'' in {\em Advances
		in neural information processing systems}, pp.~3844--3852, 2016.
	
	\bibitem{kipf2016semi}
	T.~N. Kipf and M.~Welling, ``Semi-supervised classification with graph
	convolutional networks,'' {\em arXiv preprint arXiv:1609.02907}, 2016.
	
	\bibitem{xu2019graph}
	B.~Xu, H.~Shen, Q.~Cao, Y.~Qiu, and X.~Cheng, ``Graph wavelet neural network,''
	{\em arXiv preprint arXiv:1904.07785}, 2019.
	
	\bibitem{xu2019graphheat}
	B.~Xu, H.~Shen, Q.~Cao, K.~Cen, and X.~Cheng, ``Graph convolutional networks
	using heat kernel for semi-supervised learning,'' in {\em Proceedings of the
		28th International Joint Conference on Artificial Intelligence},
	pp.~1928--1934, AAAI Press, 2019.
	
	\bibitem{zhou2004regularization}
	D.~Zhou and B.~Sch{\"o}lkopf, ``A regularization framework for learning from
	graph data,'' in {\em ICML 2004 Workshop on Statistical Relational Learning
		and Its Connections to Other Fields (SRL 2004)}, pp.~132--137, 2004.
	
	\bibitem{belkin2004regularization}
	M.~Belkin, I.~Matveeva, and P.~Niyogi, ``Regularization and semi-supervised
	learning on large graphs,'' in {\em International Conference on Computational
		Learning Theory}, pp.~624--638, Springer, 2004.
	
	\bibitem{smola2003kernels}
	A.~J. Smola and R.~Kondor, ``Kernels and regularization on graphs,'' in {\em
		Learning theory and kernel machines}, pp.~144--158, Springer, 2003.
	
	\bibitem{li2019label}
	Q.~Li, X.-M. Wu, H.~Liu, X.~Zhang, and Z.~Guan, ``Label efficient
	semi-supervised learning via graph filtering,'' in {\em Proceedings of the
		IEEE Conference on Computer Vision and Pattern Recognition}, pp.~9582--9591,
	2019.
	
	\bibitem{belkin2006manifold}
	M.~Belkin, P.~Niyogi, and V.~Sindhwani, ``Manifold regularization: A geometric
	framework for learning from labeled and unlabeled examples,'' {\em Journal of
		machine learning research}, vol.~7, no.~Nov, pp.~2399--2434, 2006.
	
	\bibitem{zhu2003semi}
	X.~Zhu, Z.~Ghahramani, and J.~D. Lafferty, ``Semi-supervised learning using
	gaussian fields and harmonic functions,'' in {\em Proceedings of the 20th
		International conference on Machine learning (ICML-03)}, pp.~912--919, 2003.
	
	\bibitem{weston2012deep}
	J.~Weston, F.~Ratle, H.~Mobahi, and R.~Collobert, ``Deep learning via
	semi-supervised embedding,'' in {\em Neural networks: Tricks of the trade},
	pp.~639--655, Springer, 2012.
	
	\bibitem{pmlr-v97-wu19e}
	F.~Wu, A.~Souza, T.~Zhang, C.~Fifty, T.~Yu, and K.~Weinberger, ``Simplifying
	graph convolutional networks,'' in {\em Proceedings of the 36th International
		Conference on Machine Learning}, pp.~6861--6871, 2019.
	
	\bibitem{klicpera2019diffusion}
	J.~Klicpera, S.~Wei{\ss}enberger, and S.~G{\"u}nnemann, ``Diffusion improves
	graph learning,'' in {\em Advances in Neural Information Processing Systems},
	pp.~13333--13345, 2019.
	
	\bibitem{gama2019stability}
	F.~Gama, J.~Bruna, and A.~Ribeiro, ``Stability properties of graph neural
	networks,'' {\em arXiv preprint arXiv:1905.04497}, 2019.
	
	\bibitem{zhu2009introduction}
	X.~Zhu and A.~B. Goldberg, ``Introduction to semi-supervised learning,'' {\em
		Synthesis lectures on artificial intelligence and machine learning}, vol.~3,
	no.~1, pp.~1--130, 2009.
	
	\bibitem{smola1998connection}
	A.~J. Smola, B.~Sch{\"o}lkopf, and K.-R. M{\"u}ller, ``The connection between
	regularization operators and support vector kernels,'' {\em Neural networks},
	vol.~11, no.~4, pp.~637--649, 1998.
	
	\bibitem{hammond2011wavelets}
	D.~K. Hammond, P.~Vandergheynst, and R.~Gribonval, ``Wavelets on graphs via
	spectral graph theory,'' {\em Applied and Computational Harmonic Analysis},
	vol.~30, no.~2, pp.~129--150, 2011.
	
	\bibitem{yang2016revisiting}
	Z.~Yang, W.~W. Cohen, and R.~Salakhutdinov, ``Revisiting semi-supervised
	learning with graph embeddings,'' in {\em Proceedings of the 33rd
		International Conference on International Conference on Machine
		Learning-Volume 48}, pp.~40--48, 2016.
	
	\bibitem{perozzi2014deepwalk}
	B.~Perozzi, R.~Al-Rfou, and S.~Skiena, ``Deepwalk: Online learning of social
	representations,'' in {\em Proceedings of the 20th ACM SIGKDD international
		conference on Knowledge discovery and data mining}, pp.~701--710, 2014.
	
	\bibitem{lu2003link}
	Q.~Lu and L.~Getoor, ``Link-based classification,'' in {\em Proceedings of the
		20th International Conference on Machine Learning (ICML-03)}, pp.~496--503,
	2003.
	
	\bibitem{kingma2014adam}
	D.~P. Kingma and J.~Ba, ``Adam: A method for stochastic optimization,'' {\em
		arXiv preprint arXiv:1412.6980}, 2014.
	
	\bibitem{glorot2010understanding}
	X.~Glorot and Y.~Bengio, ``Understanding the difficulty of training deep
	feedforward neural networks,'' in {\em Proceedings of the thirteenth
		international conference on artificial intelligence and statistics},
	pp.~249--256, 2010.
	
	\bibitem{abadi2016tensorflow}
	M.~Abadi, A.~Agarwal, P.~Barham, E.~Brevdo, Z.~Chen, C.~Citro, G.~S. Corrado,
	A.~Davis, J.~Dean, M.~Devin, {\em et~al.}, ``Tensorflow: Large-scale machine
	learning on heterogeneous distributed systems,'' {\em arXiv preprint
		arXiv:1603.04467}, 2016.
	
\end{thebibliography}
\end{document}